\documentclass{article}

    \PassOptionsToPackage{numbers, compress}{natbib}

    \usepackage[preprint]{neurips_2024}

\usepackage[utf8]{inputenc} 
\usepackage[T1]{fontenc}    
\usepackage{hyperref}       
\usepackage{url}            
\usepackage{booktabs}       
\usepackage{amsfonts}       
\usepackage{nicefrac}       
\usepackage{microtype}      
\usepackage{xcolor}         
\usepackage{xspace}
\newcommand{\method}{Consistency-FM\xspace}
\usepackage{amsmath}   
\usepackage{amsthm}
\usepackage{wrapfig}
\usepackage{subcaption}
\usepackage{graphics}
\usepackage{graphicx}
\usepackage{hyperref}
\hypersetup{
	colorlinks=true,
	linkcolor=red,
	filecolor=blue,      
	urlcolor=magenta,
	citecolor=black,
}
\usepackage{mathtools}
\newtheorem{lemma}{Lemma}
\usepackage{lipsum}
\newtheorem{theorem}{Theorem}
\newtheorem{corollary}{Corollary}[theorem]
\usepackage[capitalize]{cleveref}
\usepackage{tcolorbox}
\usepackage{adjustbox}
\usepackage{multicol}
\usepackage{multirow, booktabs}
\usepackage{longtable}
\usepackage{array}
\usepackage{fixltx2e}
\usepackage{setspace}
\usepackage{tabularx}
\usepackage{calc}
\usepackage{makecell}
\usepackage{bm}
\usepackage{colortbl}
\title{Consistency Flow Matching: \\Defining Straight Flows with Velocity Consistency}

\author{%
  Ling Yang\thanks{Corresponding author, yangling0818@163.com} \\
  Peking University\\
  \And
  Zixiang Zhang \\
  Peking University \\
  \AND
  Zhilong Zhang \\
  Peking University \\
  \And
  Xingchao Liu \\
  University of Texas at Austin \\
  \And
  Minkai Xu \\
  Stanford University \\
  \And
  Wentao Zhang \\
  Peking University \\
    \And
  Chenlin Meng \\
  Pika Labs \\
    \And
  Stefano Ermon \\
  Stanford University \\
  \And
  Bin Cui \\
  Peking University \\
}

\begin{document}

\maketitle

\begin{abstract}
Flow matching (FM) is a general framework for defining probability paths via Ordinary Differential Equations (ODEs) to transform between noise and data samples. Recent approaches attempt to straighten these flow trajectories to generate high-quality samples with fewer function evaluations, typically through iterative rectification methods or optimal transport solutions. In this paper, we introduce Consistency Flow Matching (\method), a novel FM method that explicitly enforces self-consistency in the velocity field. \method directly defines straight flows starting from different times to the same endpoint, imposing constraints on their velocity values. Additionally, we propose a multi-segment training approach for \method to enhance expressiveness, achieving a better trade-off between sampling quality and speed. Experiments demonstrate that our \method significantly improves training efficiency by converging 4.4x faster than consistency models and 1.7x faster than rectified flow models while achieving better generation quality. Our code is available at \href{https://github.com/YangLing0818/consistency_flow_matching}{https://github.com/YangLing0818/consistency\_flow\_matching}
\end{abstract}

\section{Introduction}

In recent years, deep generative models have provide an attractive family of paradigms that can produce high-quality samples by modeling a data distribution, achieving promising results in many generative scenarios, such as image generation \citep{ho2020denoising,yang2023diffusion,yang2024mastering,yang2024cross}. As a general and deterministic framework, Continuous Normalizing
Flows (CNFs) \citep{chen2018neural} are capable of modeling arbitrary probability paths, specifically including the probability paths represented by diffusion processes \citep{song2021maximum}. To scale up the training of CNFs, many works propose efficient simulation-free approaches \citep{lipman2022flow,albergo2022building,liu2022flow} by parameterizing
a vector field which flows from noise samples to data samples. Lipman et al. (2022) \cite{lipman2022flow} proposes Flow Matching
(FM) to train CNFs based on constructing explicit
conditional probability paths between the noise distribution and each data sample. Taking inspiration from denoising score matching \cite{song2019generative}, FM further shows that a per-example training objective can provide equivalent gradients without requiring explicit knowledge of the intractable target vector field, thus incorporating existing diffusion paths as special instances.

Straightness is one particularly-desired property of the trajectory induced by FM \citep{liu2022flow,liu2023instaflow,kornilov2024optimal,tong2023improving}, because the straight path are not only the shortest path between two end points, but also can be exactly simulated without time discretization.
To learn straight line paths which transport distribution $\pi_0$ to $\pi_1$, Liu et al. (2022) \cite{liu2022flow} learn a rectified flow from data by turning an arbitrary coupling of $\pi_0$ and $\pi_1$ to a
new deterministic coupling, and iteratively train new rectified flows with the data simulated from the previously obtained rectified flow. Some works resort to optimizing with an optimal transport plan by considering non-independent couplings of k-sample empirical distributions \citep{pooladian2023multisample,tong2023improving}. For example, OT-CFM \citep{tong2023improving} attempts to approximate dynamic OT, creating
simpler flows that are more stable to train and lead to faster inference.

\begin{figure}[t]
	\vskip -0.6in
	\centering
	\includegraphics[width=0.98\linewidth]{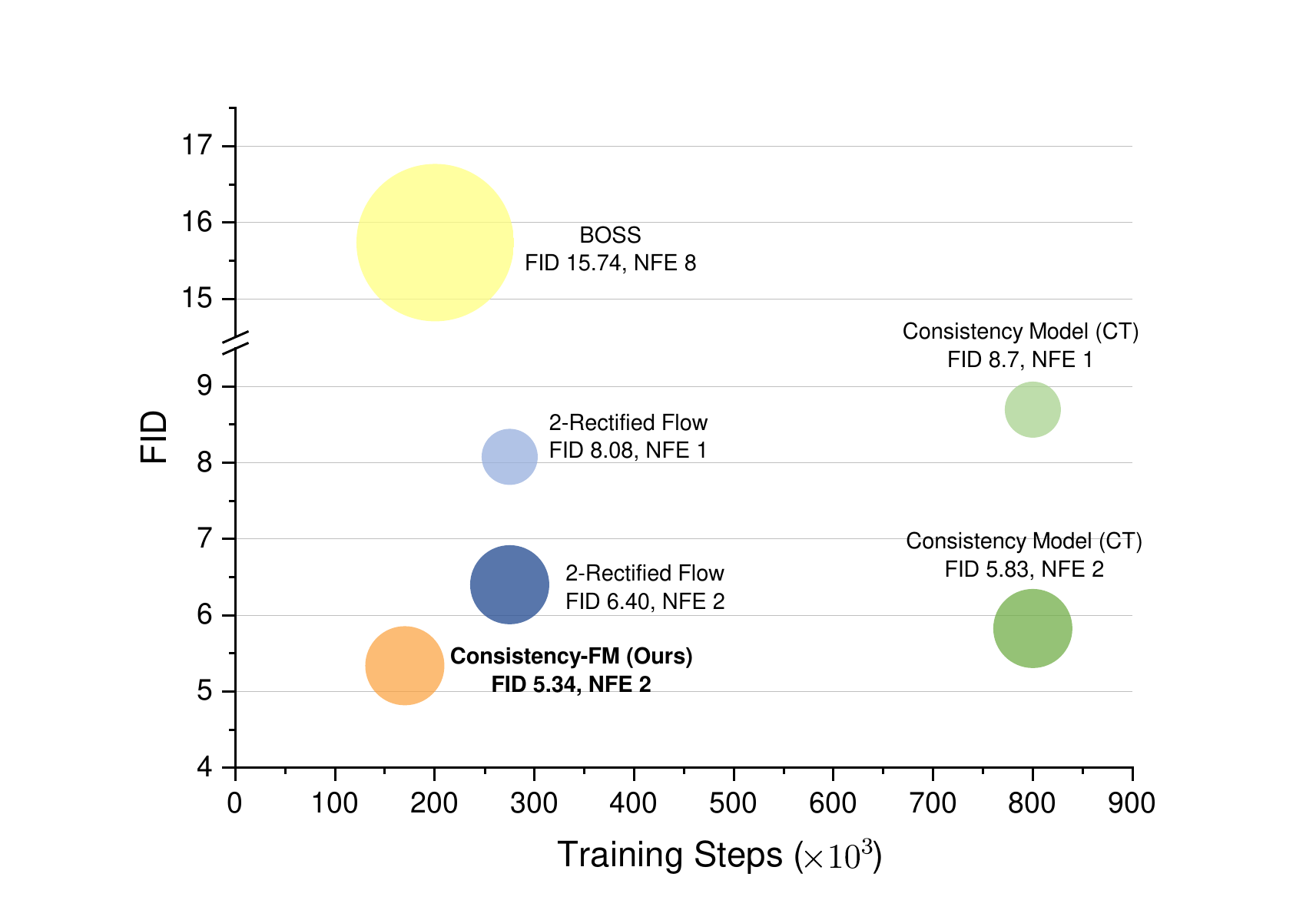}
	\caption{Comparison on CIFAR-10 dataset regarding the trade-off between generation quality and training efficiency. Our \method demonstrates the best trade-off compared to consistency models \citep{song2023consistency} and rectified flow models \citep{liu2022flow,nguyen2024bellman}, \textbf{converging 4.4 times faster than consistency models and 1.7 times faster than rectified flow models} while achieving better generation quality}
	\label{fig:efficiency_comparison}
\end{figure}

However, despite their impressive generation quality, they still lack an effective trade-off between sampling quality and computational cost in straightening flows. To be more specifically, iterative rectification would suffer from accumulation error, and approximating an optimal transport plan in each training batch is computationally expensive. Therefore, a question naturally arises, \textit{can one learn an effective ODE model that maximally straightens the trajectories of probability flows without increasing training complexity?}

In this work, we propose a new fundamental FM method, namely \textbf{Consistency Flow Matching (\method)}, to straighten the flows by  explicitly enforcing self-consistency property in the velocity field. More specifically, \method directly defines straight flows that start from different times to the same endpoint, and further constrains on their velocity values. To enhance the model expressiveness and enable better transporting between complex distributions, we resort to training \method in a multi-segment approach, which constructs a piece-wise linear trajectory. Moreover, this flexible time-to-time jump allows \method to perform distillation on pre-trained FM models for better trade-off between sampling speed and quality.  

\begin{figure}[t]
  \centering
  \includegraphics[width=0.98\linewidth]{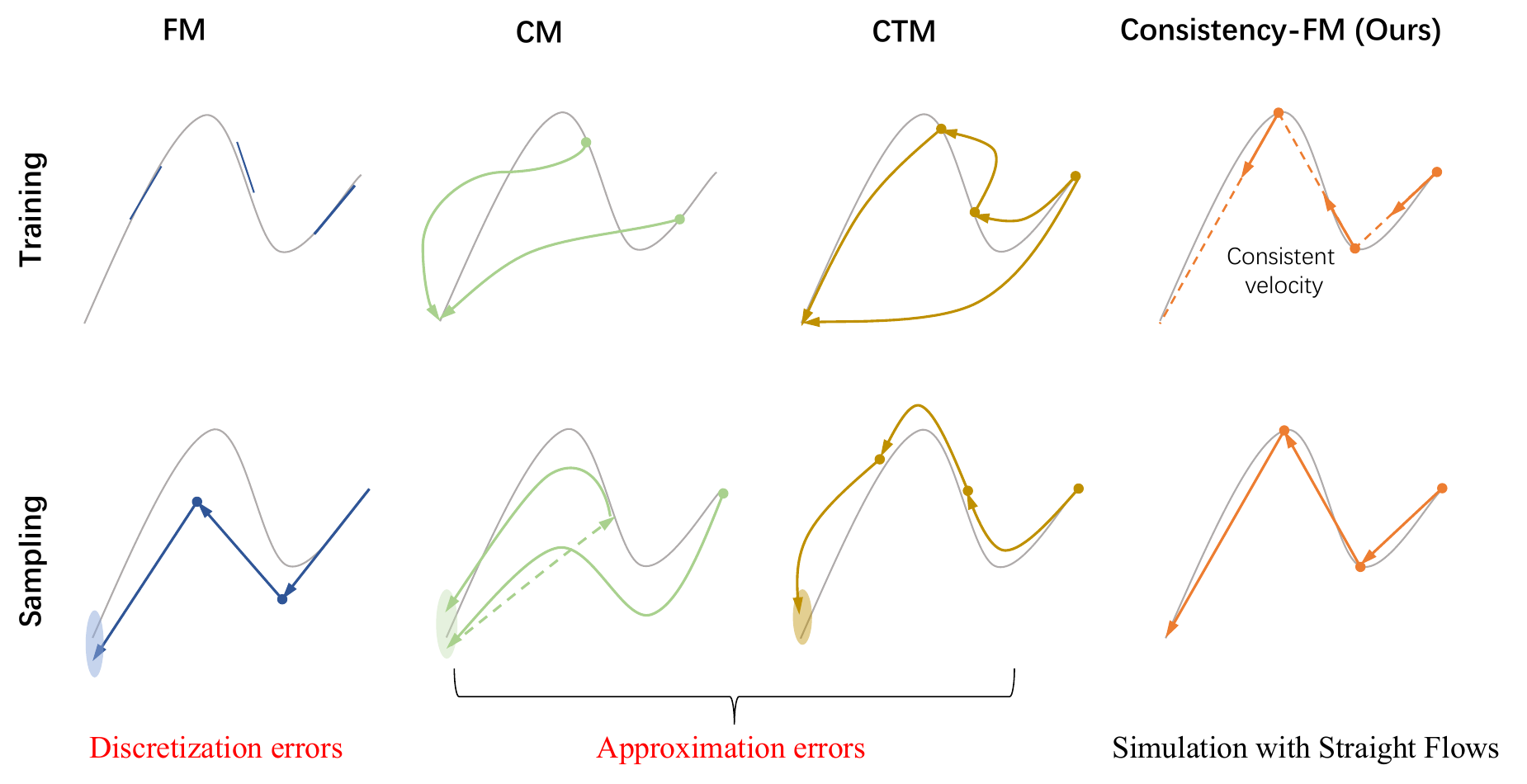}
  \caption{Training and sampling comparisons between flow matching (FM) \citep{lipman2022flow}, consistency model (CM) \citep{song2023consistency} and consistency trajectory model (CTM) \citep{kim2023consistency} and our \method. While previous methods can cause discretization errors or approximation errors, \method mitigates these issues by defining straight flows in simulation.}
  \label{fig:first_fig}
\end{figure}

\paragraph{Comparison with Consistency Models} Consistency Models (CMs) \citep{song2023consistency} 
learn a set of consistency functions that directly map noise to data. 
While CMs can generate sample with one NFE, but they fail to provide a satisfying trade-off between generation quality and computational cost \citep{kim2023consistency}. 
Moreover, enforcing consistency property at arbitrary points is redundant and potentially slows down the training process. 
In contrast, 
our \method enforces the consistency property over the space of velocity field instead of sample space, which 
can be viewed as a high-level regularization for straightening ODE trajectory. While CMs is able to learn consistency functions in a general form, \method parameterizes the consistency functions as straight flows, which enables faster training convergence without the need for approximating the entire probability path.

\paragraph{Main Contributions} We summarize our contributions as follows: (i) We propose a new fundamental class of FM models that explicitly enforces the self-consistency in the space of velocity field instead of sample space. (ii) We conduct sufficient theoretical analysis for our proposed \method, and enhance its expressiveness with multi-segment optimization.  (iii) Preliminary experiments on three classical image datasets demonstrate the superior generation quality and training efficiency of our \method (e.g., 4.4 times and 1.7 times faster than consistency model and rectified flow).

\section{Related Work and Discussions}
\paragraph{Flow Matching for Generative Modeling}
Flow Matching (FM) aims to (implicitly) learn a vector field $\{v_t\}_{t\in[0,1]}$, which generates an ODE that admits to the desired probability path $\{p_t\}_{t\in[0,1]}$ \citep{lipman2022flow}. The training of FM does not require any computational challenging simulation, as it directly estimate the vector field using a regression objective which can be efficiently estimated \citep{lipman2022flow}. By the construction of FM, it allows general trajectory of ODE and probability path, thus many effort have been dedicated to design better trajectory with certain properties\citep{pooladian2023multisample,tong2023improving,klein2023equivariant,stark2024dirichlet,campbell2024generative}. One particularly desired property is the straightness of the trajectory, as a straight trajectory can be efficiently simulated with few steps of Euler integration. Concurrent works Multisample FM \citep{pooladian2023multisample} and Minibatch OT \citep{tong2023improving} propose to generalize the independent coupling of data distribution $p_0(x_0)$ and prior distribution $p_1(x_1)$ to optimal transport coupling plan $\pi(x_0,x_1)$. Under the optimal transport plan, the learned trajectory of ODE will tend to be straight. However, their methods require constructing the approximated optimal transport plan in each training batch, which is computationally prohibitive. 

Rectified Flow \citep{liu2022flow,liu2023instaflow} can be viewed as a FM with specific trajectory. Rectified Flow proposes to rewire and straighten the trajectory by iterative distillation, which requires multiple round of training and may suffer from accumulation error.  A recent work, Optimal FM \citep{kornilov2024optimal} proposes to directly learn the optimal transport map from $p_1$ to $p_0$ and use it to calculate the vector field and straight trajectory. However, computing the optimal transport map in high dimension is a challenging task \citep{makkuva2020optimal}, and Optimal FM \citep{kornilov2024optimal} only provides experiments on toy datasets. In this paper, we propose to straighten the trajectory in a more flexible and effective approach by enforcing the self-consistency property in the velocity field. 

\paragraph{Learning Efficient Generative Models}
GANs \citep{arjovsky2017wasserstein,goodfellow2014generative}, VAEs \citep{kingma2013auto}, Diffusion Models \citep{song2020score,song2020denoising,ho2020denoising} and Normalizing Flows \citep{rezende2015variational,dinh2016density} have been four classical deep generative models. Among them, GANs and VAEs are efficient one-step models. However, GANs usually suffer from the training
instability and mode collapse issues, and VAEs may struggle to generate high-quality examples. Therefore, recent works begin to utilize diffusion models and continuous normalizing flows \citep{chen2018neural} for better training stability and high-fidelity generation, which are based on a sequence of expressive transformations for generative sampling. 

To achieve a better trade-off between sampling quality and speed, many efforts have been made to accelerate diffusion models, either by modifying the diffusion process \citep{song2020denoising,bao2021analytic,dockhorn2021score,xiao2021tackling,yang2024cross}, with an efficient ODE solver \citep{ludpm,dockhorn2022genie,zheng2023fast}, or performing distillation between pre-trained diffusion models and their more efficient versions (e.g., with less sampling steps) \citep{salimans2022progressive,liu2022flow,luo2024diff,luo2023comprehensive}. However, most distillation methods require multiple training rounds and are susceptible to accumulation errors. Recent Consistency Models \citep{song2023consistency,song2023improved} distill the entire sampling process of diffusion model into one-step generation, while maintaining good sample quality. Consistency Trajectory Models (CTMs)\citep{kim2023consistency} further mitigate the issues about the accumulated errors in multi-step sampling. However, these method must learn to integrate the full ODE integral, which are difficult to learn when it jumps between modes of the target distribution.
In this paper, we propose a new concept of \textit{velocity consistency} with defined straight probability flows, achieving most competitive results on both one- and multi-step generation. 



\section{Consistency Flow Matching}
\subsection{Preliminaries on Flow Matching}
Let $\mathcal{R}^d$ denote the data space with data point $x_0 \in \mathcal{R}^d $, FMs aim to the learn a vector field $v_\theta(t,x): [0,1]\times\mathcal{R}^d \xrightarrow{} \mathcal{R}^d$, such that the solution of the following ODE can transport noise $x_0~\sim p_0$ to data $x_1~\sim p_1$: 
\begin{equation}\label{ODE-1}
    \left\{
        \begin{aligned}
            &\frac{d \gamma_x(t)}{dt} = v_\theta(t,\gamma_x(t)), \\
            &\gamma_x(0)= x 
        \end{aligned}
    \right.
\end{equation}
The solution of \cref{ODE-1} is denoted by $\gamma_x(\cdot)$, which is also called a flow, describing the trajectory of the ODE from starting point $x$. Given the ground truth vector field $u(t,x)$ that generates probability path $p_t$ under the two marginal constraints that $p_{t=0} = p_0$ and $p_{t=1} = p_1$, FMs seek to optimize the simple regression objective
\begin{equation} \label{objective-1}
    E_{t,p_t} ||v_\theta(t,x_t)-u(t,x_t)||_2^2
\end{equation}
However, it is computational intractable to find such $u$, since $u$ and $p_t$ are governed by the following continuity equation \citep{villani2009optimal}:
\begin{equation} \label{continuity}
    \partial_tp_t(x) = - \nabla \cdot (u(t,x)p_t(x))
\end{equation}
Instead of directly optimizing \cref{objective-1}, Conditional Flow Matching \citep{lipman2022flow} regress $v_\theta(t,x)$ on the conditional vector filed $u(t,x_t|x_1)$ and probability path $p_t(x_t|x_1)$ :
\begin{equation} \label{objective-2}
    E_{t,q(x_1)}E_{p_t(x_t|x_1)} ||v_\theta(t,x_t)-u(t,x_t|x_1)||_2^2
\end{equation}
Two objectives \cref{objective-1} and \cref{objective-2} share the same gradient with respect to $\theta$, while \cref{objective-2} can be efficiently estimated as long as the conditional pair $u(t,x_t|x_1), p_t(x_t|x_1)$ is tractable. Note that recovering the marginal vector field and probability path from the conditioned one remains a complex challenge \citep{lipman2022flow}.

\subsection{Defining Straight Flows with Consistent Velocity}
\paragraph{Motivation} Recent FM methods for learning straight flows typically necessitate the approximation the probability path $p_t$ and its marginal distributions $p_0$ and $p_1$   \citep{liu2022flow,liu2023instaflow,pooladian2023multisample,lee2023minimizing} , which are computational intensive and introduce additional approximation error. To address these challenges, we introduce \method, a general method to efficiently learn straight flows without the need for approximating the entire probability path.

A straightforward approach to learn straight flows is to identify a consistent ground truth vector field and then use objective in \cref{objective-1} for training. The definition of consistent velocity is $v(t,\gamma_x(t)) =v(0,x)$, indicating the velocity along the solution of \cref{ODE-1} remains constant. However, due to the intractability of original data distribution, it is also intractable to find such a vector field, or to design a conditional vector field such that the corresponding marginal velocity is consistent \citep{lipman2022flow,pooladian2023multisample}.

\begin{figure}[t]
  \centering
  \includegraphics[width=0.9\linewidth]{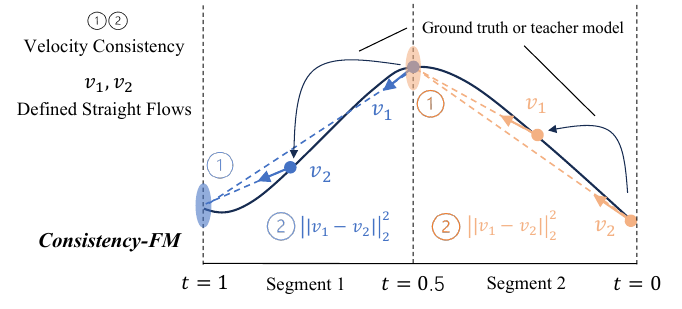}
  \caption{Ilustration of training our consistency-FM.}
  \label{fig:method}
\end{figure}

Instead of directly regressing on the ground truth vector field, \method directly defines straight flows with consistent velocity that start from different times to the same endpoint. Specifically, we have the following lemma (prove in Appendix \ref{pf_lemma1}):
\begin{lemma} \label{lemma1}
    Assuming the vector field is Lipschitz with respect to $x$ and uniform in $t$, and are differentiable in both input, then these two conditions are equivalent:
    \begin{equation}
        \begin{aligned}
            &\textit{Condition 1.} \quad v(t,\gamma_x(t)) =v(s,\gamma_x(s)),\quad \forall t,s\in[0,1]\\
            &\textit{Condition 2.} \quad \gamma_{x}(t)+(1-t)*v(t,\gamma_x(t))= \gamma_{x}(s)+(1-s)*v(s,\gamma_x(s)), \quad \forall t,s \in [0,1], 
        \end{aligned}
    \end{equation}
\end{lemma}
where $\gamma_{x}(t)$ represents the solution of \cref{ODE-1} at time $t$. \textit{Condition 2} specifies that starting from an arbitrary time $t$ with data point $ \gamma_{x}(t)$, and moving in the direction of current velocity for a duration of $1-t$, the resulting data will be consistent and independent with respect to $t$. 

\paragraph{Velocity Consistency Loss} While \textit{Condition 1} directly constraints the vector field to be consistent, learning vector fields that only satisfy \textit{Condition 1} may lead to trivial solutions.  On the other hand, \textit{Condition 2} ensures the consistency of the vector field from a trajectory viewpoint, offering a way to directly define straight flows. Motivated by this, \method learns a consistency vector field to satisfy both conditions: 
\begin{equation}\label{loss-1}
    \begin{aligned}
        & \mathcal{L}_{\theta} = E_{t\sim \mathcal{U}}E_{x_t,x_{t+\Delta t}} || f_\theta(t,x_{t}) - f_{\theta^-}(t+\Delta t,x_{t+\Delta t})||_2^2+ \alpha|| v_\theta(t,x_{t}) - v_{\theta^-}(t+\Delta t,x_{t+\Delta t})||_2^2, \\
        & f_\theta(t,x_t) = x_t + (1-t)*v_\theta(t,x_t), \\
    \end{aligned}
\end{equation}
where $\mathcal{U}$ is the uniform distribution on $[0,1-\Delta t]$, $\alpha$ is a positive scalar, $\Delta t$ denotes a time interval which is a small and positive scalar. $\theta^-$ denotes the running average of past values of $\theta$ using  exponential moving average (EMA), $x_t$ and $x_{t+\Delta t}$ follows a pre-defined distribution which can be efficiently sampled, for example, VP-SDE \citep{ho2020denoising} or OT path \citep{lipman2022flow}. Note that by setting $t=1$, \textit{Condition 2} implies that $\gamma_{x}(t)+(1-t)*v(t,\gamma_x(t)) = \gamma_{x}(1) \sim p_1$, and thus training with $\mathcal{L}_\theta$ can not only regularize the velocity but also learn the data distribution. Furthermore, if \textit{Condition 2} is met, then  the straight flows $\gamma_{x}(t)+(1-t)*v(t,\gamma_x(t))$ can directly predict $x_1$ from each time point $t$ \citep{song2023consistency}. 

Below we provide a theoretical justification for the objective based on asymptotic analysis (proof in Appendix \ref{pf_theorem1}).
\begin{theorem}\label{theorem-1}
    Consider no exponential moving average, i.e., $\theta^- = \theta$. Assume there exists ground truth velocity field $u_t$ that generates $p_t$ and satisfies the continuity  \cref{continuity}. Furthermore we assume $v_\theta$ is bounded and twice continuously differentiable with bounded first and second derivatives, the ground truth velocity $u_t$ is bounded. Then we have:
    \begin{equation}\label{connectFM}
        E|| f_\theta(t,x_{t}) - f_{\theta}(t+\Delta t,x_{t+\Delta t})||_2^2 = (\Delta t)^2 E||v_\theta(t,x_t)-u(t,x_t)-(1-t)(\partial_t v_\theta + u \cdot \nabla_x v_\theta)||_2^2 + o((\Delta t)^2)
    \end{equation}
\end{theorem}
\textbf{Remark 1.} The objective in \cref{connectFM},
\begin{equation*}
    E||v_\theta(t,x_t)-u(t,x_t)-(1-t)(\partial_t v_\theta + u \cdot \nabla_x v_\theta)||_2^2,
\end{equation*}
can be seen as striking a balance between exact velocity estimation and adhering to consistent velocity constraints. On the one hand, the objective aims to minimize the discrepancy between learned and ground truth velocity $v_\theta(t,x_t)-u(t,x_t)$, aligning with the goal of FM-based methods \citep{lipman2022flow}. On the other hand, it also considers the consistency of the velocity. By Lemma \ref{lemma2} in the Appendix, $\partial_t v_\theta + u \cdot \nabla_x v_\theta$ serves as a constraint for velocity consistency, which measures the changes of the velocity after taking a infinitesimal step along the direction of ground truth velocity. Given the ground truth velocity may not be consistent, this objective provides a trade-off between the sampling quality and computational cost with straight flow.

\subsection{Multi-Segment \method}
To enhance the expressiveness of \method for transporting distributions in general probability path, we introduce Multi-Segment \method. This approach relaxes the requirement for consistent velocity throughout the flow, allowing for more flexible adaptations to diverse distribution characteristics. Multi-Segment \method divides the time interval into equal segments, learning a consistent vector field $v_\theta^i$ within each segment. After recombining these segments, it constructs a piece-wise linear trajectory to transport noise to data distribution. Specifically, given a segment number $K$, the time interval $[0,1]$ is divided with $[0,1] = \Sigma_{i=0}^{K-1}[i/K,(i+1)/K]$. Then the training objective is defined as
\begin{equation}\label{loss-2}
    \begin{aligned}
        &\mathcal{L}_{\theta} = E_{t\sim \mathcal{U}^i}\lambda^iE_{x_t,x_{t+\Delta t}} || f_\theta^i(t,x_t) - f_{\theta^-}^i(t+\Delta t,x_{t+\Delta t})||_2^2+ \alpha|| v_\theta^i(t,x_t) - v_{\theta^-}^i(t+\Delta t,x_{t+\Delta t})||_2^2, \\
        &f_\theta^i(t,x_t) = x_t + ((i+1)/k-t)*v_\theta^i(t,x_t),\\
    \end{aligned}
\end{equation}
where $i$ denotes the $i^{th}$ segment,  $\mathcal{U}^i$ is the uniform distribution on $[i/K,(i+1)/K-\Delta t]$, $x_{t},x_{t+\Delta t}$ follow a pre-defined distribution, $\Delta t$ is a small and positive constant . $v_\theta^i(t,x_{t})$ are the flow and the consistent vector field in segment $i$, respectively. $\lambda^i$ is a positive weighting scalar for different segment, as vector field in the middle of $[0,1]$ is more difficult to train \citep{esser2024scaling}. 

Below we provide a theoretical justification for multi-segment training.
First, we analysis the optimal solution for objective \cref{loss-2} and provide a explicit formula for the estimation error in Multi-Segment Consistency-FM (proof in Appendix \cref{pf_theorem2}): 
\begin{theorem}\label{theorem 2}
    Consider training consistency-FM on segment $i$ which is defined in time interval $[S,T]$. Assume there exists ground truth velocity field $u_t$ that generates $p_t$ and satisfies the continuity equation \cref{continuity}, let $v^*(t,x)$ denote the oracle consistent velocity such that 
    \begin{equation}
        x_T = x_t + (T-t)v^*(t,x_t).
    \end{equation}
    Then the learned $v_\theta^i(t,x_t)$ that minimize \cref{loss-2} in segment $i$ at time $t \in [S,T - \Delta t]$ has the following error:
    \begin{equation}\label{error-1}
    \begin{aligned}
        v_\theta^i(t,x_t) - v^*(t,x_t) =& \frac{\alpha}{(T-t)^2+\alpha}(v^*(t+\Delta t, x_{t+\Delta t})-v^*(t, x_{t})) \\
        & + \frac{(T-t-\Delta t)(T-t) +\alpha}{(T-t)^2+\alpha}(v_\theta^i(t+\Delta t,x_{t+\Delta t}) - v^*(t+\Delta t,x_{t+\Delta t}))
    \end{aligned}
    \end{equation}
\end{theorem}

\textbf{Remark 2.} The mismatch between the learned velocity $v_\theta^i(t,x_t)$ and oracle velocity $v^*(t,x)$ is composed of two parts. The first part is the inconsistency of the oracle $v^*(t+\Delta t, x_{t+\Delta t})-v^*(t, x_{t})$, which is due to the fact that the ground truth velocity $u(t,x_t)$ might not be consistent. If $u(t,x_t)$ is consistent within the time interval $[S,T]$, then the oracle velocity is the ground truth velocity $v^*(t,x) = u(t,x_t)$ and $v^*(t,x)$ is consistent, and thus the error in the first part will vanish. The second part is the accumulated error from prior time step $t+\Delta t$, and by induction, we can deduce that this part will also vanish if the $u(t,x_t)$ is consistent. As a result, Consistency-FM can learn the ground truth velocity with objective \cref{objective-2} on any time interval $[S,T]$ where the ground truth velocity is consistent.

\begin{corollary}\label{theorem 3}
    Consider training consistency-FM on segment $i$ which is defined in time interval $[S,T]$. Assume there exists ground truth velocity field $u_t$ that generates $p_t$ and satisfies the continuity equation \cref{continuity}. If the ground truth velocity $u$ is consistent within $[S,T]$, then Consistency-FM can learn the ground truth velocity, i.e., the learned $v_\theta(t,x_t) = u(t,x_t)$ almost everywhere.
\end{corollary}

\subsubsection{Distillation with \method}
\method can also be trained with pre-trained FMs. For distillation from a pre-trained FM $u_\phi(t,x_t)$, the consistency distillation loss for \method is defined as 
\begin{equation}
    \begin{aligned}
        & \mathcal{L}_{\theta,\phi} = E_{t \sim \mathcal{U}}E_{x_{t}}||f_\theta(t,x_t) - f_{\theta^-}(t+\Delta t,\hat{x}_{t+\Delta t}^\phi)||_2^2+\alpha||v_\theta(t,x_t) - v_{\theta^-}(t+\Delta t,\hat{x}_{t+\Delta t}^\phi)||_2^2, \\
        & f_\theta(t,x_t) = x_t + (1-t)*v_\theta(t,x_t), \\
        & \hat{x}_{t+\Delta t}^\phi = x_t + \Delta t *u_{\phi}(t,x_t), \\
    \end{aligned}
\end{equation}
where $\mathcal{U}[0,1-\Delta t]$ is the uniform distribution, $u_{\phi}(t,x)$ is the pre-trained FM, $x_{t}$ follows the distribution from which $u_\phi$ is trained, $\hat{x}_{t+\Delta t}^\phi$ is the one-step prediction using pre-trained model. For distillation from a pre-trained FMs, we set the segment number $K=1$, as evidences show that the flows in pre-trained FMs are relatively straight \citep{liu2022flow,pooladian2023multisample}.

\subsubsection{Sampling with \method}
\method facilitates both one-step and multi-step generation. With a well-trained Consistency FM $v_\theta(\cdot,\cdot)$ ,  we can generate sample by sampling from prior distribution $x_0 = p_0$ and then evaluating the model to transport the data through $k$ segments:
\begin{equation}
    x_{i/k} = x_{(i-1)/k} + 1/k * v_\theta^i((i-1)/k,x_{(i-1)/k}), i=1,2,\dots k-1
\end{equation}
Alternatively, iterative sampling can be employed with the standard Euler method within each segment \citep{butcher2016numerical}:
\begin{equation}
    x_{t+\Delta t} = x_{t} + \Delta t * v_\theta^i(t,x_{t}), t \in [i/k,(i+1)/k-\Delta t]
\end{equation}
This approach offers a versatile framework that facilitates a balanced trade-off between sample quality and sampling efficiency.

\section{Experiments}

\paragraph{Experimental Settings}
We evaluate our \method on unconditional image generation tasks for preliminary experiments. Specifically, we use the CIFAR-10 \citep{alex2009learning} and two high-resolution (256x256) datasets CelebA-HQ \citep{karras2017progressive} and AFHQ-Cat \citep{choi2020stargan}. 
To make a fair comparison with previous methods, we evaluate the sample procedure with NFE = \{2, 6, 8\}, adopt the U-Net architecture of DDPM++ \citep{song2020score} as Rectified Flow \citep{liu2022flow} and use Frechet inception distance (FID) score \citep{heusel2017gans} to measure the quality of generated image samples. All models are initialized randomly, and more experimental settings can be found in \cref{tab:exp_detail1}.

\begin{table}[ht]
    \caption{Experimental details for training \method.}\label{tab:exp_detail1}
    \vspace{0.06in}
    \centering
    {\setlength{\extrarowheight}{1.5pt}
    \begin{adjustbox}{max width=\linewidth}
    \begin{tabular}{l|ccc}
        \Xhline{3\arrayrulewidth}
        Training Details & CIFAR-10 & AFHQ-Cat 256 × 256 & CelebA-HQ 256 × 256 \\
        \Xhline{3\arrayrulewidth}
        Training iterations & 180k & 250k & 250k \\
        Batch size & 512 & 64 & 64 \\
        Optimizer & Adam & Adam & Adam \\
        Learning rate & 2e-4 & 2e-4 & 2e-4 \\
        EMA decay rate & 0.999999 & 0.999 & 0.999 \\
        Dropout probability & 0.0 & 0.0 & 0.0 \\
        ODE solver & Euler & Euler & Euler \\
        \Xhline{3\arrayrulewidth}
    \end{tabular}
    \end{adjustbox}
    }
\end{table}
\paragraph{Baseline Methods} To demonstrate the effectiveness of our \method, we follow previous work \citep{song2023consistency} and compare \method with some representative diffusion models and flow models, such as Consistency Model \citep{song2023consistency} and Rectified Flow \citep{liu2022flow}. In the experiments on  AFHQ-Cat and CelebA-HQ datasets, we also add recent Bellman Sampling \citep{nguyen2024bellman} for flow matching models as the baseline. 

\subsection{\method Beats Rectified Flow and Consistency Model}
As demonstrated in \cref{tab:results}, on CIFAR-10 dataset, \method’s NFE 2 FID (5.34) not only surpasses representative efficient generative models like Consistency Model (5.83) and Rectified Flow (378) on unconditional generation, but also is comparable to those mainstream diffusion models with NFE > 30. Notably, in \cref{fig:efficiency_comparison}, our \method significantly advances training efficiency, \textbf{converging 4.4 times faster than consistency model and 1.7 times faster than rectified flow} while achieving superior sampling quality. These evaluation results sufficiently show that our \method provides a more efficient way to model data distribution, proving the efficacy of our proposed learning paradigm of velocity consistency for FM models.
\begin{table}[ht]
\vspace{-0.03in}
	\caption{Comparing \method with previous diffusion models and flow models on CIFAR-10.}\label{tab:results}
	\centering
        \vspace{0.06in}
	\begin{tabular}{lccc}
        \Xhline{3\arrayrulewidth}
	    Method & NFE ($\downarrow$) & FID ($\downarrow$) & IS ($\uparrow$) \\
        \\[-2ex]
        \Xhline{3\arrayrulewidth}
        Score SDE \cite{song2020score} & 2000 & 2.20 & \textbf{9.89}\\
        DDPM \cite{ho2020denoising} & 1000 & 3.17 & 9.46\\
        LSGM \cite{vahdat2021score} & 147 & 2.10 & -\\
        PFGM \cite{xu2022poisson} & 110 & 2.35 & 9.68\\
        EDM \cite{karras2022edm}
         & 35 & \textbf{2.04} & 9.84 \\
         \hline
        1-Rectified Flow \cite{liu2022flow}
         & 1 & 378 & 1.13\\
        Glow \cite{kingma2018glow}
         & 1 & 48.9 & 3.92 \\
        Residual Flow \cite{chen2019residual} & 1 & 46.4&-\\
        GLFlow \cite{xiao2019generative} & 1 & 44.6 & -\\
        DenseFlow \cite{grcic2021densely} & 1 & 34.9 & -\\
        Consistency Model  \citep{song2023consistency} & 2 & 5.83 & \textbf{8.85} \\
        \hline
        \textbf{Consistency Flow Matching}&2&\textbf{5.34} &\textbf{8.70}\\
        \Xhline{3\arrayrulewidth}
	\end{tabular}
\vspace{-0.1in}
\end{table}

\subsection{High-Resolution Image Generation}
\cref{tab:results2} shows the quantitative result of FM models and our \method on high-resolution (256 × 256) image generation, including AFHQ-Cat and CelebA-HQ. We can observe that our \method also outperform existing SOTA FM methods like rectified flow and rectified flow + Bellman sampling \citep{nguyen2024bellman} by a significant margin with same NFEs. Furthermore, compared to CIFAR-10, \method shows a greater improvement in generating high-resolution images. This phenomenon demonstrates that our \method can potentially learn straighter flows for modeling more complex data distribution, enabling faster and better sampling. 
\begin{table}[ht]
	\caption{Comparing \method with previous flow matching models.}\label{tab:results2}
 \vspace{0.06in}
	\centering
	{\setlength{\extrarowheight}{1.5pt}
	\begin{adjustbox}{max width=\linewidth}
	\begin{tabular}{l|cccccc}
	    \Xhline{3\arrayrulewidth}
        Method&\multicolumn{2}{l}{\textbf{AFHQ-Cat $\bm{256\times 256}$}}&\multicolumn{2}{l}{\textbf{CelebA-HQ $\bm{256\times 256}$}}\\
         & NFE ($\downarrow$) & FID ($\downarrow$)& NFE ($\downarrow$) & FID ($\downarrow$) \\
        \Xhline{3\arrayrulewidth}
        Rectified Flow \citep{liu2022flow} & 8 &57.0&8&109.4\\
        Rectified Flow + Bellman Sampling \citep{nguyen2024bellman} & 8 &33.9&8&49.8\\
        Rectified Flow \citep{liu2022flow} & 6 &61.5&6&127.0\\
        Rectified Flow + Bellman Sampling \citep{nguyen2024bellman} & 6 &36.2&6&72.5\\
        \hline
        \textbf{Consistency Flow Matching}&6&\textbf{22.5}&6&\textbf{36.4}\\
        \Xhline{3\arrayrulewidth}
	\end{tabular}
    \end{adjustbox}
    }
    \vspace{-0.2in}
\end{table}

\subsection{Qualitative Analysis}
We provide three convergence processes of training our \method in \cref{fig:convergence}. We observe that \method converges faster on CIFAR-10 than on AFHQ-Cat and CelebA-HQ because the latter two high-resolution datasets are more complex to model their data distributions. Overall, \method consistently converges fast, proving the efficacy of defining straight flows for generative modeling. Additionally, we qualitatively compare our method with rectified flow in \cref{fig:qualitative-cifar}. From the generation results, we can observe that our \method is capable of generating more realistic images than rectified flow with the same NFEs, revealing our \method models data distribution more effectively.  
\begin{figure}[ht]
\vskip -0.1in
	\centering
	\begin{subfigure}{0.31\linewidth}
		\centering
		\includegraphics[width=0.98\linewidth]{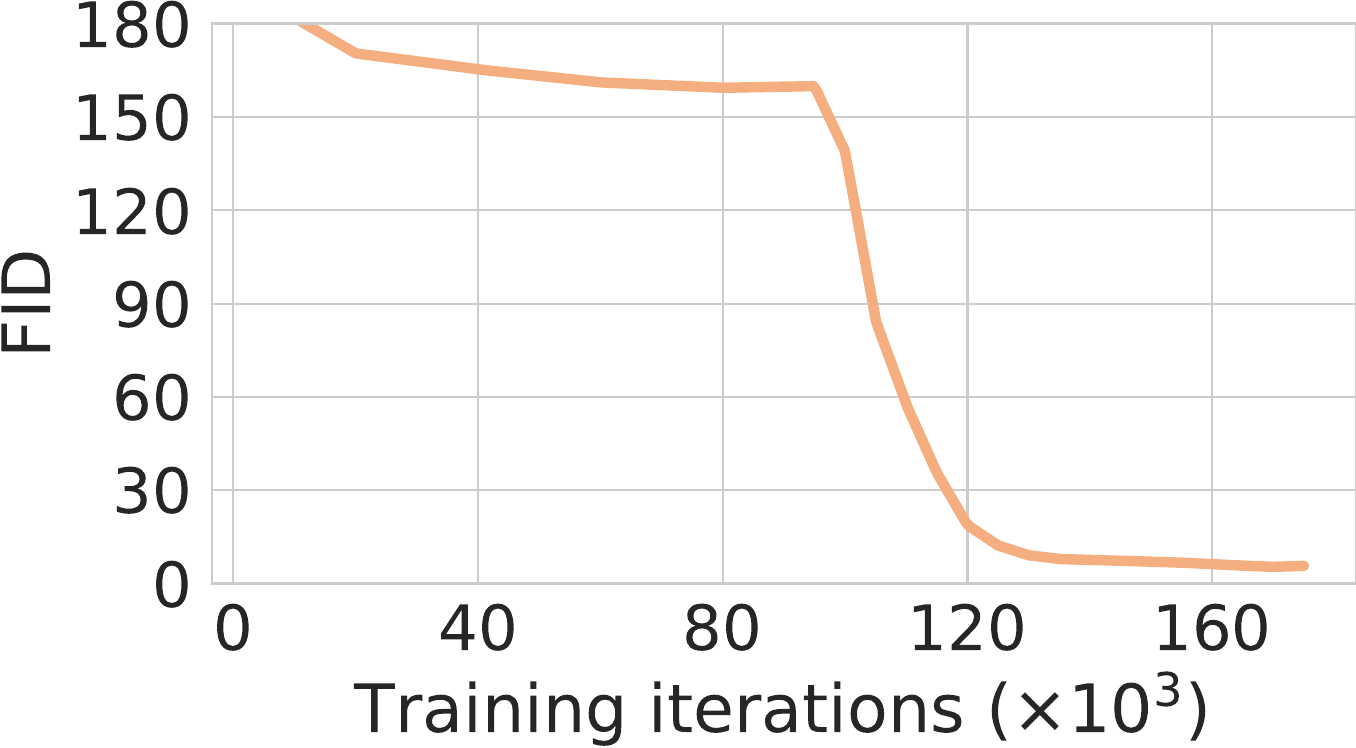}
		\caption{CIFAR-10, NFE 2}
	\end{subfigure}
	\begin{subfigure}{0.31\linewidth}
		\centering
		\includegraphics[width=0.98\linewidth]{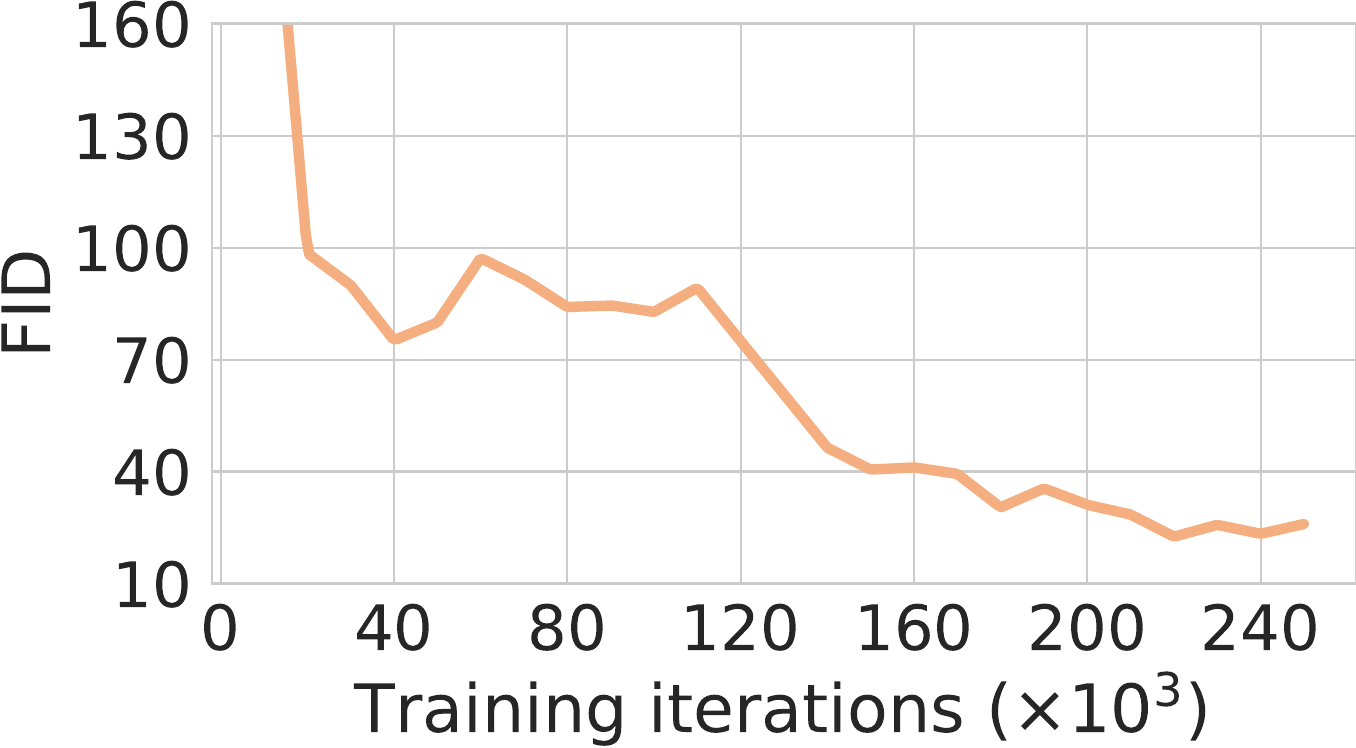}
		\caption{AFHQ-Cat, NFE 6}
	\end{subfigure}	
 	\begin{subfigure}{0.31\linewidth}
		\centering
		\includegraphics[width=0.98\linewidth]{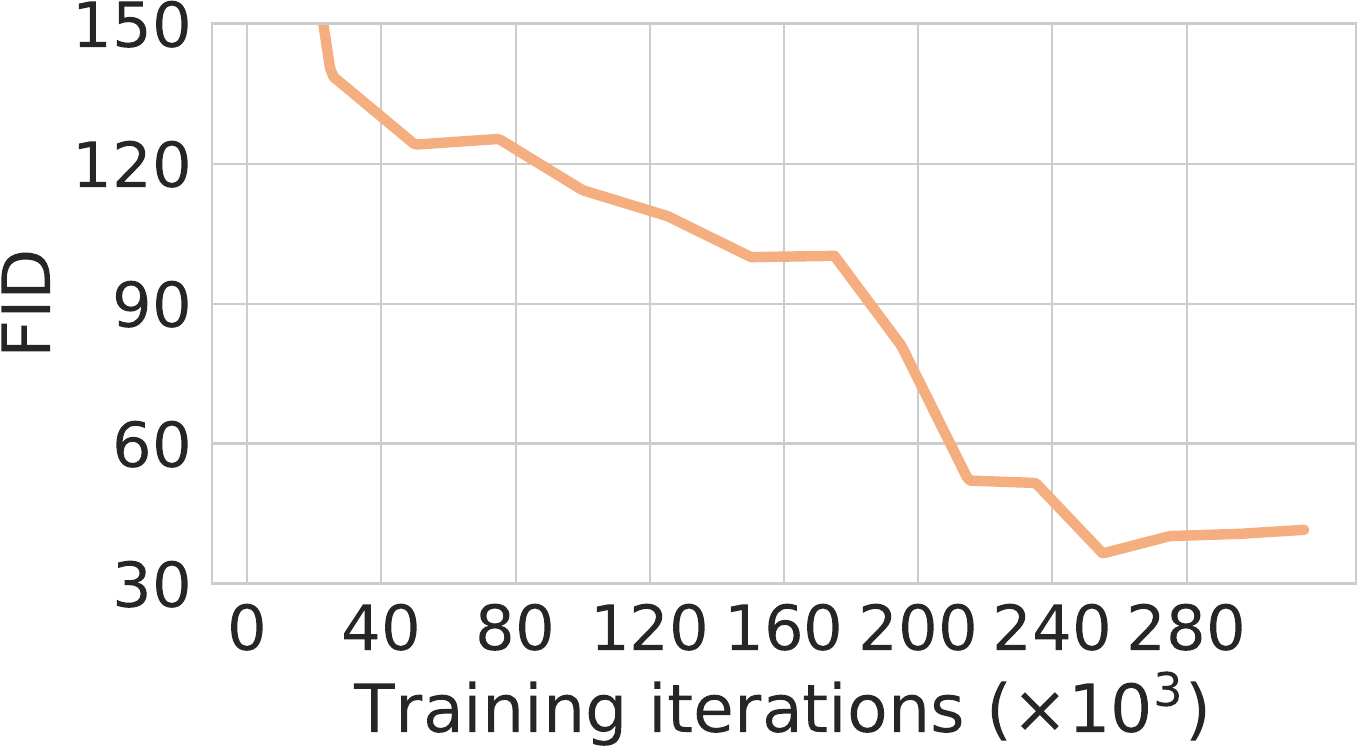}
		\caption{CelebA-HQ, NFE 6}
	\end{subfigure}
	\caption{Demonstration of training convergence on three datasets.}
    \label{fig:convergence}
\end{figure}

\begin{figure}[ht]
	\centering
	\includegraphics[width=0.9\linewidth]{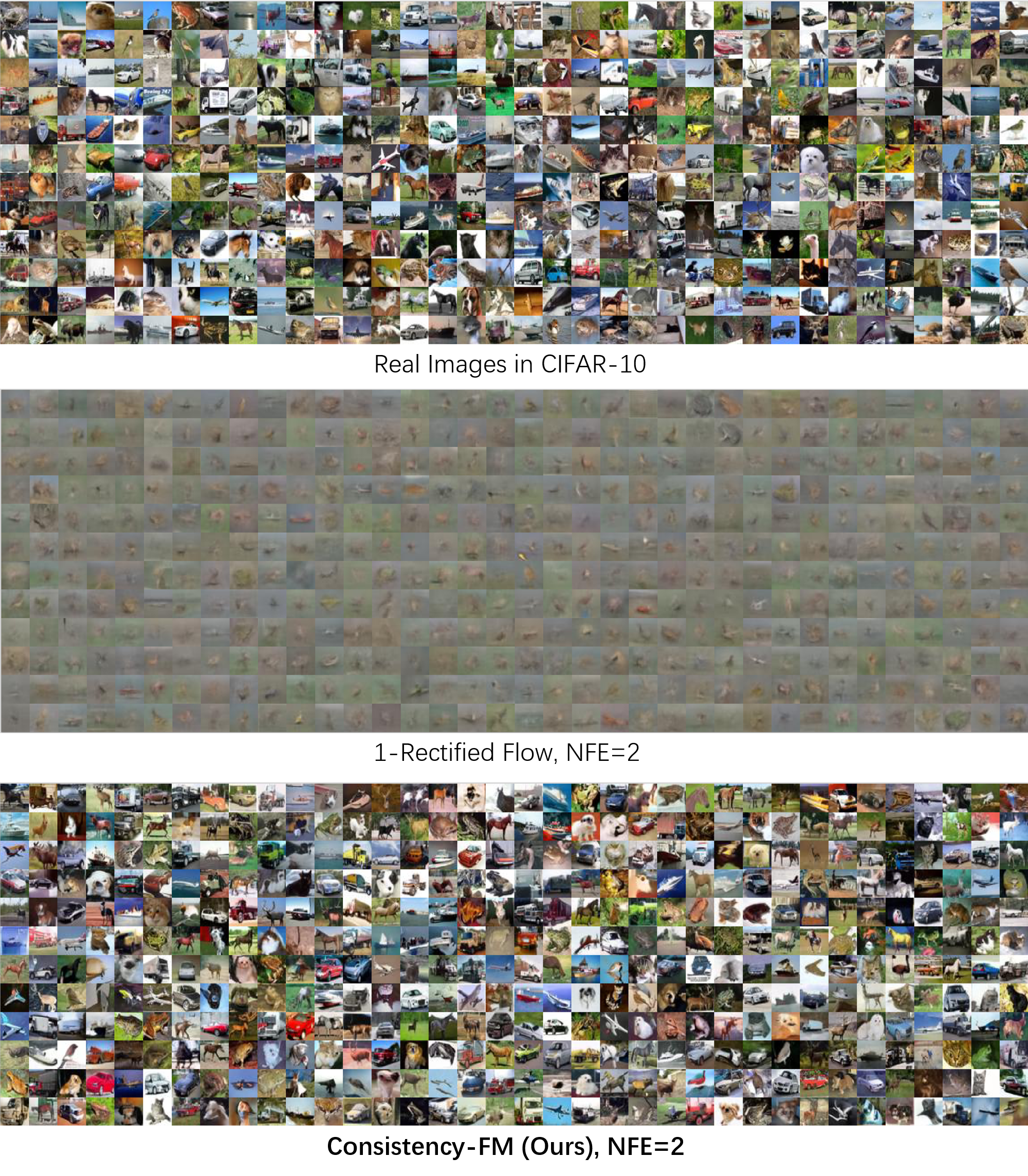}
	\caption{Sampling comparison between Rectified Flow \citep{liu2022flow} and our \method.}
	\label{fig:qualitative-cifar}
\end{figure}

\section{Future Work}
This work theoretically and empirically presents our new fundamental flow matching model. While we achieve some improvements in generated image quality and training efficiency. We here discuss a few directions for advancing this research:
\begin{itemize}
    \item \textbf{Text-to-image generation:} While we demonstrates superior performance and efficiency on unconditional image generation, extending to large-scale datasets and more complex generation scenarios (e.g. text-to-image generation \citep{rombach2022high,podell2023sdxl,yang2024mastering}) is necessary to further improve the capacity of our proposed velocity consistency. 
    \item \textbf{Distillation with pretrained models:} In method part, we provide a loss function to distill pretrained FMs with our \method. From a more general perspective, it is worth to explore whether our \method is able to distill both pretrained DMs and FMs.
\end{itemize}

{\small
\bibliographystyle{ieeetr}
\bibliography{neurips_2024}
}

\newpage 
\appendix

\section{Theoretical Supports and Proofs}

\subsection{PROOF OF LEMMA \ref{lemma1}}\label{pf_lemma1}
\begin{proof}[Proof of Lemma \ref{lemma1}] 
If \textit{Condition 1} is meet, then ODE \cref{ODE-1} associated with $v$ becomes
\begin{equation*}
    \frac{d \gamma_x(t)}{dt} = v(t,\gamma_x(t)) = v(0,x),
\end{equation*}
and the solution of which is $\gamma_x(t) = x + t*v(0,x)$. Specifically, we have
\begin{equation*}
    \begin{aligned}
            \gamma_x(1) &= x + 1*v(0,x) \\
            &= \gamma_x(t) + (1-t)*v(0,\gamma_x(0))= \gamma_x(t) + (1-t)*v(t,\gamma_x(t))\\
            &= \gamma_x(s) + (1-s)*v(0,\gamma_x(0)) =\gamma_x(s) + (1-s)*v(t,\gamma_x(s))
    \end{aligned}
\end{equation*}
and thus \textit{Condition 2} is meet.

On the other hand, if \textit{Condition 2} is meet, then we have
\begin{equation*}
    \begin{aligned}
            \gamma_x(t) - \gamma_x(s) &= (1-s)v(s,\gamma_x(s))-(1-t)v(t,\gamma_x(t)) \\
            &= \int_{s}^{t} v(u,\gamma_x(u))du\\
    \end{aligned}
\end{equation*}
Divide both hands in the above equation with $t-s$ and let $t$ approaches $s$, we have:

\begin{equation}
    \begin{aligned}
        & v(s,\gamma_x(s))= \lim_{t\to s}\frac{\int_{s}^{t} v(u,\gamma_x(u))du}{t-s},\\
        & = \lim_{t\to s} \frac{(1-s)v(s,\gamma_x(s))-(1-t)v(t,\gamma_x(t))}{t-s}\\
        &=\lim_{t\to s} v(t,\gamma_x(t)) + (1-s) * \frac{v(s,\gamma_x(s))-v(t,\gamma_x(t))}{t-s} = v(s,\gamma_x(s)) - (1-s) \frac{d v(s,\gamma_s)}{ds}
    \end{aligned}
\end{equation}

Comparing the both sides in the above equation, we have $\frac{d v(s,\gamma_s)}{ds}=0$, and thus $\textit{Condition 1}$ is meet.
\end{proof}

\subsection{PROOF OF LEMMA \ref{lemma2} } \label{pr_lemma2}
We provide an another lemma which describes the consistency constraint as partial differential equations and supports the connection bewteen Consistency-FM with FMs.
\begin{lemma}\label{lemma2}
    Assume $v$ is continuously differentiable and consistent, then $v$ satisfies the following equation:
    \begin{equation}
        \partial_t v(t,x) + v \cdot \nabla_x v = 0
    \end{equation}
\end{lemma}

\begin{proof}[Proof of Lemma \ref{lemma2}]
    By Lemma \ref{lemma1}, if $v$ is consistent then it can be written as:
    \begin{equation}
        v(t+\Delta t, x_t+\Delta t*v(t,x_t)) = v(t,x_t)
    \end{equation}
    Since $v$ is differentiable, we can take derivatives with respect to $t$:
    \begin{equation}
        \frac{d v}{dt} = \partial_t v + v \cdot \nabla_x v
    \end{equation}
    Then by definition, if $v$ is consistent we have
    \begin{equation}
        \frac{d v}{dt} = \partial_t v + v \cdot \nabla_x v = 0
    \end{equation} 
\end{proof}
\subsection{PROOF OF THEOREM \ref{theorem-1}} \label{pf_theorem1}
\begin{proof}[Proof of Theorem \ref{theorem-1}] By the first mean value theorem, there exist a $t^{'} \in [t,t+\Delta t]$, such that 
\begin{equation}
    x_{t+\Delta t}- x_t = \int_{t}^{t+\Delta t} u(s,x_s)du = \Delta t * u(t^{'}, x_{t^{'}})
\end{equation}
Then by the definition of $f_\theta(t,x_t) = x_t +(1-t)*v_\theta(t,x_t)$, we have:
    \begin{equation}
    \begin{aligned}
    &f_\theta(t,x_{t}) - f_{\theta}(t+\Delta t,x_{t+\Delta t}) \\
    & =x_t - x_{t+\Delta t} + (1-t)*v_\theta(t,x_t)-(1-t-\Delta t)*v_\theta(t+\Delta t,x_{t+\Delta t}), \\
    & = \Delta t *v_\theta(t+\Delta t,x_{t+\Delta t})-\int_{t}^{t+\Delta t}u(s,x_s)ds + (1-t)*(v_\theta(t,x_t)-v_\theta(t+\Delta t,x_{t+\Delta t})),\\
    & =^1 \Delta t *(v_\theta(t+\Delta t,x_{t+\Delta t})-u(t^{'},x_{t^{'}})) + (1-t)*(v_\theta(t,x_t)-v_\theta(t+\Delta t,x_{t+\Delta t})),\\
    & =^2 \Delta t * [((v_\theta(t+\Delta t,x_{t+\Delta t})-u(t',x_{t'})) - (1-t)(\partial_t v_\theta(t,x_t) + u(t',x_t') \cdot \nabla_x v_\theta(t,x_t))]+ O((\Delta t)^2)\\
    & =^3 \Delta t * [(v_\theta(t,x_{t})-u(t,x_{t})) - (1-t)(\partial_t v_\theta(t,x_t) + u(t,x_t) \cdot \nabla_x v_\theta(t,x_t))]+ O((\Delta t)^2), 
    \end{aligned}
    \end{equation}
    where in $(1)$ we used the first mean value theorem, in $(2)$ we used first-order Taylor approximation and in $(3)$ we used the boundedness of $v_\theta$, derivatives of $v_\theta$ and $u$. Then the objective can be written as:

    \begin{equation*}
    \begin{aligned}
        &E|| f_\theta(t,x_{t}) - f_{\theta}(t+\Delta t,x_{t+\Delta t})||_2^2,\\
        &= E||\Delta t *(v_\theta(t,x_{t})-u(t,x_{t}) - (1-t)(\partial_t v_\theta + u \cdot \nabla_x v_\theta))+ O((\Delta t)^2)||_2^2 \\
        & =(\Delta t)^2 * E||v_\theta(t,x_t)-u(t,x_t) - (1-t)(\partial_t v_\theta + u \cdot \nabla_x v_\theta)||_2^2 + o((\Delta t)^2)
    \end{aligned}
    \end{equation*}
    
\end{proof}

\subsection{PROOF OF THEOREM \ref{theorem 2}} \label{pf_theorem2}
    \begin{proof}[Proof of Theorem \ref{theorem 2}]
    As $v_\theta^i(t,x_t)$ is the minimizer of \cref{loss-2} with respect to segment $i$, it must satisfies the first-order condition: 
    \begin{equation}\label{first-order}
        \begin{aligned}
            0 &= \partial_\theta (E|| f_\theta^i(t,x_t) - f_{\theta^-}^i(t+\Delta t,x_{t+\Delta t})||_2^2+ \alpha|| v_\theta^i(t,x_t) - v_{\theta^-}^i(t+\Delta t,x_{t+\Delta t})||_2^2)\\
            &= E ((T-t)(f_\theta^i(t,x_t) - f_{\theta^-}^i(t+\Delta t,x_{t+\Delta t})) + \alpha(v_\theta^i(t,x_t) - v_{\theta^-}^i(t+\Delta t,x_{t+\Delta t}))) \cdot \partial_\theta v_\theta^i(t,x_t),
        \end{aligned}
    \end{equation}
    where $f_\theta^i(t,x_t) = x_t +(T-t)v_\theta^i(t,x_t)$
    Note that in our assumption, $x_{t+\Delta t}$ is generated by an ODE and thus is a deterministic function of $(t,x_t)$, then the non-trivial solution to \ref{first-order} satisfies the following equation almost everywhere:
    \begin{equation}\label{first-order1}
        0 = (T-t)(f_\theta^i(t,x_t) - f_{\theta^-}^i(t+\Delta t,x_{t+\Delta t})) + \alpha(v_\theta^i(t,x_t) - v_{\theta^-}^i(t+\Delta t,x_{t+\Delta t})),
    \end{equation}
    As the gradient at $\theta$ is zero, then $\theta = \theta^-$, thus the learned velocity can be derived from \ref{first-order1}:
    \begin{equation}
        v_\theta^i(t,x_t) = \frac{(T-t)(x_{t+\Delta t}-x_t)}{(T-t)^2+\alpha} +\frac{(T-t-\Delta t)(T-t)+\alpha}{(T-t)^2+\alpha}v_{\theta}^i(t+\Delta t,x_{t+\Delta t})
    \end{equation}
    Furthermore, we have:
    \begin{equation}
    \begin{aligned}
        &v_\theta^i(t,x_t) \\
        &= \frac{(T-t)(x_{t+\Delta t}-x_t)}{(T-t)^2+\alpha} +\frac{(T-t-\Delta t)(T-t)+\alpha}{(T-t)^2+\alpha}v^*(t+\Delta t,x_{t+\Delta t})\\
        & \quad + \frac{(T-t-\Delta t)(T-t)+\alpha}{(T-t)^2+\alpha}v_{\theta}^i(t+\Delta t,x_{t+\Delta t}) - \frac{(T-t-\Delta t)(T-t)+\alpha}{(T-t)^2+\alpha}v^*(t+\Delta t,x_{t+\Delta t})\\
        &= \frac{(T-t)((T-(t+\Delta t))v^*(t+\Delta t,x_{t+\Delta t})+x_{t+\Delta t}-x_t)}{(T-t)^2+\alpha}+\frac{\alpha v^*(t+\Delta t,x_{t+\Delta t})}{(T-t)^2+\alpha}\\
        & \quad + \frac{(T-t-\Delta t)(T-t)+\alpha}{(T-t)^2+\alpha}(v_{\theta}^i(t+\Delta t,x_{t+\Delta t})-v^*(t+\Delta t,x_{t+\Delta t}))\\
        &=^1  \frac{(T-t)(x_T-x_t)}{(T-t)^2+\alpha} +\frac{\alpha v^*(t+\Delta t,x_{t+\Delta t})}{(T-t)^2+\alpha}\\
        & \quad + \frac{(T-t-\Delta t)(T-t)+\alpha}{(T-t)^2+\alpha}(v_{\theta}^i(t+\Delta t,x_{t+\Delta t})-v^*(t+\Delta t,x_{t+\Delta t}))\\
        &=^2  \frac{(T-t)^2v^*(t,x_{t})}{(T-t)^2+\alpha} +\frac{\alpha v^*(t+\Delta t,x_{t+\Delta t})}{(T-t)^2+\alpha}\\
        & \quad + \frac{(T-t-\Delta t)(T-t)+\alpha}{(T-t)^2+\alpha}(v_{\theta}^i(t+\Delta t,x_{t+\Delta t})-v^*(t+\Delta t,x_{t+\Delta t}))\\
        &= v^*(t,x_{t}) + \frac{\alpha(v^*(t+\Delta t,x_{t+\Delta t}) - v^*(t,x_{t}))}{(T-t)^2+\alpha}\\
        & \quad + \frac{(T-t-\Delta t)(T-t)+\alpha}{(T-t)^2+\alpha}(v_{\theta}^i(t+\Delta t,x_{t+\Delta t})-v^*(t+\Delta t,x_{t+\Delta t}))\\
    \end{aligned}
    \end{equation}
    where in (1) and (2) we have use the assumption of oracle $v^*$ that $x_T = x_t + (T-t)v^*(t,x_t)$
    \end{proof}

\subsection{PROOF OF COROLLARY \ref{theorem 3}} \label{pf_corollary1}
\begin{proof}[Proof for Corollary \ref{theorem 3}]
    Note that $x_T = x_T + 0*v_\theta^i(T,x_{T})$, and thus we can set arbitrary value for $v_\theta^i(T,x_{T})$ without affecting the model. Specifically, we set $v_\theta^i(T,x_{T}) = v^*(T,x_{T})$. Then by Theorem \ref{theorem 2}, the error at $T-\Delta t$ is:
    \begin{equation}
        \begin{aligned}
            &v_\theta^i(T-\Delta t,x_{T-\Delta t}) - v^*(T-\Delta t,x_{T-\Delta t}) \\
            & = \frac{\alpha}{(\Delta t)^2+\alpha}(v^*(T, x_{T})-v^*(T-\Delta t, x_{T-\Delta t})) \\
            & \quad + \frac{(\Delta t-\Delta t)(T-t) +\alpha}{(\Delta t)^2+\alpha}(v_\theta^i(T,x_{T}) - v^*(T,x_{T}))\\
            &=\frac{\alpha}{(\Delta t)^2+\alpha}(v^*(T, x_{T})-v^*(T-\Delta t, x_{T-\Delta t}))
        \end{aligned}
    \end{equation}
     Furthermore, as $u$ is consistent within $[S,T]$, by Lemma \ref{lemma1} we have :
     \begin{equation*}
        \begin{aligned}
            &x_T = x_t + (T-t) * u(t,x_t),\\
            & \Rightarrow v^*(t,x_t) = u(t,x_t) \equiv u(T,x_T),\\
            &\Rightarrow v^*(t,x_t) = v^*(t+\Delta t,x_{t+\Delta t}), \forall t \in [S,T-\Delta t].
        \end{aligned}
     \end{equation*} And thus $v_\theta^i(T-\Delta t,x_{T-\Delta t}) - v^*(T-\Delta t,x_{T-\Delta t}) = 0$. 
     
     By Theorem \ref{theorem 2}, we can deduce by induction that 
     \begin{equation*}
        \begin{aligned}
            &v_\theta^i(t+\Delta t,x_{t+\Delta t}) = v^*(t+\Delta t,x_{t+\Delta t})\quad \& \quad v^*(t,x_t) = v^*(t+\Delta t,x_{t+\Delta t})\\
            &\Rightarrow v_\theta^i(t,x_t) = v^*(t,x_t) = u(t,x_t), \forall t \in [S,T-\Delta t].
        \end{aligned}
     \end{equation*}
\end{proof}

\end{document}